\documentclass[nohyperref]{article}

\usepackage{microtype}
\usepackage{graphicx}
\usepackage{subfigure}
\usepackage{booktabs} 

\usepackage{hyperref}

\usepackage[accepted]{icml2023}

\usepackage{amsmath}
\usepackage{amssymb}
\usepackage{mathtools}
\usepackage{amsthm}
\usepackage{physics}
\usepackage{xcolor}         
\usepackage{thmtools}
\usepackage{thm-restate}

\usepackage[capitalize,noabbrev]{cleveref}

\definecolor{mydarkblue}{rgb}{0,0.08,0.45}
\hypersetup{ %
    colorlinks=true,
    linkcolor=mydarkblue,
    citecolor=mydarkblue,
    filecolor=mydarkblue,
    urlcolor=mydarkblue,
}

\theoremstyle{plain}
\newtheorem{theorem}{Theorem}[section]

\newtheorem{lemma}[theorem]{Lemma}

\theoremstyle{definition}

\theoremstyle{remark}
\newtheorem{remark}[theorem]{Remark}

\DeclareMathOperator*{\argmin}{argmin}

\newcommand{\Gauss}{\mathcal{N}}

\newcommand{\E}{\mathbb{E}}
\newcommand{\Prob}{\mathbb{P}}
\newcommand{\R}{\mathbb{R}}
\newcommand{\N}{\mathbb{N}}
\newcommand{\supp}{\mathrm{supp}}
\newcommand{\Id}{\mathrm{I}}

\newcommand{\polylog}{\mathrm{polylog}}
\newcommand{\inrprod}[2]{\left\langle #1, #2 \right\rangle}
\newcommand\numberthis{\addtocounter{equation}{1}\tag{\theequation}}
\numberwithin{equation}{section}

\usepackage[textsize=tiny]{todonotes}

\icmltitlerunning{On the Benefits of Masking for Sparse Coding}

\begin{document}

\twocolumn[
\icmltitle{Hiding Data Helps: On the Benefits of Masking for Sparse Coding}



\icmlsetsymbol{equal}{*}

\begin{icmlauthorlist}
\icmlauthor{Muthu Chidambaram}{duke}
\icmlauthor{Chenwei Wu}{duke}
\icmlauthor{Yu Cheng}{brown}
\icmlauthor{Rong Ge}{duke}
\end{icmlauthorlist}

\icmlaffiliation{duke}{Department of Computer Science, Duke University}
\icmlaffiliation{brown}{Department of Computer Science, Brown University}

\icmlcorrespondingauthor{Muthu Chidambaram}{muthu@cs.duke.edu}

\icmlkeywords{Self-Supervision, Masking, Sparse Coding, Dictionary Learning, Overparameterization, Over-realization}

\vskip 0.3in
]



\printAffiliationsAndNotice{}  

\begin{abstract}
Sparse coding, which refers to modeling a signal as sparse linear combinations of the elements of a learned dictionary, has proven to be a successful (and interpretable) approach in applications such as signal processing, computer vision, and medical imaging. While this success has spurred much work on provable guarantees for dictionary recovery when the learned dictionary is the same size as the ground-truth dictionary, work on the setting where the learned dictionary is larger (or \textit{over-realized}) with respect to the ground truth is comparatively nascent. Existing theoretical results in this setting have been constrained to the case of noise-less data. We show in this work that, in the presence of noise, minimizing the standard dictionary learning objective can fail to recover the elements of the ground-truth dictionary in the over-realized regime, regardless of the magnitude of the signal in the data-generating process. Furthermore, drawing from the growing body of work on self-supervised learning, we propose a novel masking objective for which recovering the ground-truth dictionary is in fact optimal as the signal increases for a large class of data-generating processes. We corroborate our theoretical results with experiments across several parameter regimes showing that our proposed objective also enjoys better empirical performance than the standard reconstruction objective.
\end{abstract}

\section{Introduction}
Modeling signals as sparse combinations of latent variables has been a fruitful approach in a variety of domains, and has been especially useful in areas such as medical imaging \citep{zhang2017medical}, neuroscience \citep{olshausen2004sparse}, and genomics \citep{tibshirani2008spatial}, where learning parsimonious representations of data is of high importance. The particular case of modeling data in some high-dimensional space $\R^d$ as sparse \textit{linear} combinations of a set of $p$ vectors in $\R^d$ (referred to as a \textit{dictionary}) has received significant attention over the past two decades, leading to the development of many successful algorithms and theoretical frameworks.

In this case, the typical assumption is that we are given data $y_i$ generated as $y_i \sim Az_i + \epsilon_i$, where $A \in \R^{d \times p}$ is a ground truth dictionary, $z_i$ is a sparse vector, and $\epsilon_i$ is some potentially non-zero noise. When the dictionary $A$ is known a priori, the goal of modeling is to recover the sparse representations $z_i$, and the problem is referred to as \textit{compressed sensing}. However, in many applications we do not have access to the ground truth $A$, and instead hope to simultaneously learn a dictionary $B$ that approximates $A$ along with learning sparse representations of the data.

This problem is referred to as \textit{sparse coding} or \textit{sparse dictionary learning}, and is the focus of this work. One of the primary goals of analyses of sparse coding is to provide provable guarantees on how well one can hope to recover the ground truth dictionary $A$, both with respect to specific algorithms and information theoretically. Prior work on such guarantees has focused almost exclusively on the setting where the learned dictionary $B$ also belongs to $\R^{d \times p}$ (same space as the ground truth), which is in line with the fact that recovery error is usually formulated as some form of the Frobenius norm of the difference between $B$ and $A$.

Unfortunately, in practice, one does not necessarily have access to the structure of $A$, and it is thus natural to consider what happens (and how to formulate recovery error) when learning a $B \in \R^{d \times p'}$ with $p' \neq p$. Of particular interest is the case where $p' > p$, where it is possible to recover $A$ as a sub-dictionary of $B$.

The study of this \textit{over-realized} setting was recently taken up in the work of \citet{overrealized}, in which the authors showed (perhaps surprisingly) that a modest level of over-realization can be empirically and theoretically beneficial.

However, the results of \citet{overrealized} are restricted to the noise-less setting where data is generated simply as $y_i \sim Az_i$. We thus ask the following questions:

\begin{quote}
\em
Does over-realized sparse coding run into pitfalls when there is noise in the data-generating process? And if so, is it possible to prevent this by designing new sparse coding algorithms?
\end{quote}

\subsection{Main Contributions and Outline}
In this work, we answer both of these questions in the affirmative. After providing the necessary background on sparse coding in Section \ref{prelim}, we show in Theorem \ref{overfit} of Section \ref{main} that, as intuition would lead one to suspect, using standard sparse coding algorithms for learning over-realized dictionaries in the presence of noise leads to overfitting. In fact, our result shows that even if we allow the algorithms access to infinitely many samples and allow for solving NP-hard optimization problems, the learned dictionary $B$ can still fail to recover $A$.

The key idea behind this result is that existing approaches to sparse coding rely largely on a two-step procedure (outlined in Algorithm \ref{genalgo}) of solving the compressed sensing problem $B\hat{z} = y_i$ for a learned dictionary $B$, and then updating $B$ based on a reconstruction objective $\norm{y_i - B\hat{z}}^2$. However, because we force $\hat{z}$ to be sparse, by choosing $B$ to have columns that correspond to linear combinations of the columns of $A$, we can effectively ``cheat'' and get around the sparsity constraint on $\hat{z}$. In this way, it can be optimal for reconstructing the data $y_i$ to not recover $A$ as a sub-dictionary of $B$.

On the other hand, we show in Theorem \ref{masking} that for a large class of data-generating processes, it is possible to prevent this kind of cheating in $B$ by performing the compressed sensing step on a subset of the dimensions $y_i$ and computing the reconstruction loss on the complement of that subset. This is the idea of \textit{masking} that has seen great success in large language modeling \citep{devlin-etal-2019-bert}, and our result shows that it can lead to provable benefits even in the context of sparse coding.

Finally, in Section \ref{experiments} we conduct experiments comparing the standard sparse coding approach to our masking approach across several parameter regimes. In all of our experiments, we find that the masking approach leads to better ground truth recovery, with this being more pronounced as the amount of over-realization increases.


\subsection{Related Work} \label{relwork}
\textbf{Compressed Sensing.} The seminal works of \citet{candestao1}, \citet{candestao2}, and \citet{donohocs} established conditions on the dictionary $A \in \R^{d \times p}$, even in the case where $p \gg d$ (the \textit{overcomplete} case), under which it is possible to recover (approximately and exactly) the sparse representations $z_i$ from $Az_i + \epsilon_i$. In accordance with these results, several efficient algorithms based on convex programming \citep{troppconvex, bregmanconvex}, greedy approaches \citep{omptropp, donohostable, lars}, iterative thresholding \citep{daubechies, donohothresh}, and approximate message passing \citep{amp, amp2018} have been developed for solving the compressed sensing problem. There has also been work on modifying these approaches to include a cross-validation step \citep{boufounos2007, ward2009}, which is similar to the idea of our masking objective. For comprehensive reviews on the theory and applications of compressed sensing, we refer the reader to the works of \citet{candesrev} and \citet{duarterev}.

\textbf{Sparse Coding.} Different framings of the sparse coding problem exist in the literature \citep{krausesparse, convexsparse, NIPS2009_cfecdb27}, but the canonical formulation involves solving a non-convex optimization problem. Despite this hurdle, a number of algorithms \citep{engan, ksvd, onlinesparse, arora13, arora14, arora15} have been established to (approximately) solve the sparse coding problem under varying conditions, dating back at least to the groundbreaking work of \citet{OLSHAUSEN19973311} in computational neuroscience. A summary of convergence results and the conditions required on the data-generating process for several of these algorithms may be found in Table 1 of \citet{sparsespurious}.

In addition to algorithm-specific analyses, there also exists a complementary line of work on characterizing the optimization landscape of dictionary learning. This type of analysis is carried out by \citet{sparsespurious} in the general setting of an overcomplete dictionary and noisy measurements with possible outliers, extending the previous line of work of \citet{AHARON200648}, \citet{schnass}, and \citet{geng}.

However, as mentioned earlier, these theoretical results rely on learning dictionaries that are the same size as the ground truth. To the best of our knowledge, the over-realized case has only been studied by \citet{overrealized}, and our work is the first to analyze over-realized sparse coding in the presence of noise.

\textbf{Self-Supervised Learning.} Training models to predict masked out portions of the input data is an approach to self-supervised learning that has led to strong empirical results in the deep learning literature \citep{devlin-etal-2019-bert, NEURIPS2019_dc6a7e65, NEURIPS2020_1457c0d6, He_2022_CVPR}. This success has spurred several theoretical studies analyzing how and why different self-supervised tasks can be used to improve model training \citep{tsai2020self, lee2021predicting, tosh2021contrastive}. The most closely related works to our own in this regard have studied the use of masking objectives in autoencoders \citep{cao2022understand, pan2022towards} and hidden Markov models \citep{wei2021pretrained}.

\section{Preliminaries and Setup}\label{prelim}
We first introduce some notation that we will use throughout the paper.

\textbf{Notation.} Given $n \in \N$, we use $[n]$ to denote the set $\{1, 2, ..., n\}$. For a vector $x$, we write $\norm{x}$ for the $\mathcal{L}_2$-norm of $x$ and $\norm{x}_0$ for the number of non-zeros in $x$. We say a vector $x$ is $k$-sparse if $\norm{x}_0 \le k$ and we use $\supp{(x)}$ to denote the support of $x$. For a vector $x \in \R^d$ and a set $S \subseteq [d]$, we use $[x]_S \in \R^{|S|}$ to denote the restriction of $x$ to those coordinates in $S$.

For a matrix $A$, we use $A_i$ to denote the $i$-th column of $A$. We write $\norm{A}_F$ for the Frobenius norm of $A$, and $\norm{A}_{op}$ for the operator norm of $A$, and we write $\sigma_{\min}(A)$ and $\sigma_{\max}(A)$ for the minimum and maximum singular values of $A$. For a matrix $A \in \R^{d \times q}$ and $S \subseteq [q]$, we use $A_S \in \R^{d \times |S|}$ to refer to $A$ restricted to the columns whose indices are in $S$. We use $\Id_d$ to denote the $d \times d$ identity matrix. Finally, for $M \subseteq [d]$, we use $P_M \in \R^{\abs{M} \times d}$ to refer to the matrix whose action on $x$ is $P_M x = [x]_M$. Note that for a $d\times q$ matrix $A$, $P_M A$ would give a subset of rows of $A$, which is different from the earlier notation $A_S$ which gives a subset of columns.

\subsection{Background on Sparse Coding}\label{sparsebgrnd}
We consider the sparse coding problem in which we are given measurements $y \in \R^d$ generated as $Az + \epsilon$, where $A \in \R^{d \times p}$ is a ground-truth dictionary, $z \in \R^p$ is a $k$-sparse vector distributed according to a probability measure $\Prob_z$, and $\epsilon \in \R^d$ is a noise term with i.i.d. entries. The goal is to use the measurements $y$ to reconstruct a dictionary $B$ that is as close as possible to the ground-truth dictionary $A$.

In the case where $B$ has the same dimensions as $A$, one may want to formulate this notion of ``closeness'' (or recovery error) as $\norm{A - B}_F^2$. However, directly using the Frobenius norm of $(A-B)$ is too limited, as it is sufficient to recover the columns of $A$ up to permutations and sign flips. Therefore, a common choice of recovery error \citep{sparsespurious, arora15} is the following:
\begin{align*}
    \min_{P \in \Pi} \norm{A - BP}_F^2 \numberthis \label{permutedist}
\end{align*}
where $\Pi$ is the set of orthogonal matrices whose entries are $0$ or $\pm 1$.

In the over-realized setting, when $B \in \R^{d \times p'}$ with $p' > p$, Equation \eqref{permutedist} no longer makes sense as $A$ and $B$ do not have the same size. In this case, one can generalize Equation \eqref{permutedist} to measure the distance between each column of $A$ and the column closest to it in $B$ (up to change of sign). This notion of recovery was studied by \citet{overrealized}, and we use the same formulation in this work:
\begin{align*}
    d_R(A, B) \triangleq \frac{1}{p} \sum_{i = 1}^p \min_{\substack{j \in [p'], c \in \{-1, 1\}}} \norm{A_i - cB_j}^2 \numberthis \label{otherdist}
\end{align*}
Note that Equation \eqref{otherdist} introduced the coefficient $1/p$ in the recovery error and thus corresponds to the \textit{average} distance between $A_i$ and its best approximation in $B$. 
Also, Equation~\eqref{otherdist} only allows sign changes, even though for reconstructing $Az$, it is sufficient to recover the columns of $A$ up to arbitrary scaling. In our experiments we enforce $A$ and $B$ to have unit column norms so a sign change suffices; in theory one can always modify the $B$ matrix to have correct norm so it also does not change our results.

Given access to only measurements $y$, the algorithm cannot directly minimize the recovery error $d_R(A, \cdot)$. Instead, sparse coding algorithms often seek to minimize the following surrogate loss:
\begin{align*}
    \ell(B) = \E_y \left[\min_{\hat{z} \in \R^{p'}} \norm{y - B \hat{z}} ^ 2 + h(\hat{z})\right] \numberthis \label{recon1}
\end{align*}
where $h$ is a sparsity-promoting penalty function. Typical choices of $h$ include hard sparsity ($h(\hat{z}) = 0$ if $\hat{z}$ is $k$-sparse and $h(\hat{z}) = \infty$ otherwise) as well as the $\mathcal{L}_1$ penalty $h(\hat{z}) = \norm{\hat{z}}_1$. While hard sparsity is closer to the assumption on the data-generating process, it is well-known that optimizing under exact sparsity constraints is NP-hard in the general case \citep{l0hard}. When $h(\hat{z}) = \norm{\hat{z}}_1$ is used, the learning problem is also known as basis pursuit denoising \citep{chen1994basis} or Lasso \citep{lasso}.

Equation \eqref{recon1} is the population loss one wishes to minimize when learning a dictionary $B$. In practice, sparse coding algorithms must work with a finite number of measurements $y_1, y_2, \ldots, y_n$ obtained from the data-generating process and instead minimize the empirical loss $\Tilde{\ell}(B)$:
\begin{align*}
    \Tilde{\ell}(B) = \sum_{i = 1}^n \min_{\hat{z} \in \R^{p'}} \norm{y_i - B \hat{z}} ^ 2 + h(\hat{z}) \numberthis \label{emprecon1}
\end{align*}

\subsection{Sparse Coding via Orthogonal Matching Pursuit}\label{sparsealgo}
Most existing approaches for optimizing Equation \eqref{emprecon1} can be categorized under the general alternating minimization approach described in Algorithm \ref{genalgo}. For simplicity we state Algorithm \ref{genalgo} in terms of a single input signal $y \in \R^d$, but in practice the dictionary update in Algorithm \ref{genalgo} is performed after batching over several input signals. 

\begin{algorithm}[ht]
   \caption{Alternating Minimization Framework}
\begin{algorithmic}\label{genalgo}
   \STATE {\bfseries Input:} Data $y \in \R^d$, Dictionary $B^{(t)} \in \R^{d \times p'}$
   \STATE {\bfseries Decoding Step: } Solve $B^{(t)} \hat{z} = y$ for $k$-sparse $\hat{z}$ 
   \STATE {\bfseries Update Step: } Update $B^{(t)}$ to $B^{(t + 1)}$ by performing a gradient step on loss computed using $B^{(t)}\hat{z}$ and $y$
\end{algorithmic}
\end{algorithm}

At iteration $t$, Algorithm~\ref{genalgo} performs a decoding/compressed sensing step using the current learned dictionary $B^{(t)}$ and the input data $y$. As mentioned in Section \ref{relwork}, there are several well-studied algorithms for this decoding step. Because we are interested in enforcing a hard-sparsity constraint, we restrict our attention to algorithms that are guaranteed to produce a $k$-sparse representation in the decoding step. 

We thus focus on Orthogonal Matching Pursuit (OMP) \citep{omporiginal, Rubinstein2008EfficientIO}, which is a simple greedy algorithm for the decoding step. The basic procedure of OMP is to iteratively expand a subset $T \subset [p']$ of atoms (until $\abs{T} = k$) by considering the correlation between the unselected atoms in the current dictionary $B^{(t)}$ and the residual $\left(y - B^{(t)}_T \argmin_{\hat{z} \in \R^{\abs{T}}} \norm{y - B^{(t)}_T \hat{z}}^2\right)$ (i.e., the least squares solution using atoms in $T$). A more precise description of the algorithm can be found in \citet{Rubinstein2008EfficientIO}. Moving forward, we will use $g_{\textnormal{OMP}}(y, B, k)$ to denote the $k$-sparse vector $\hat z \in \R^{p'}$ obtained by running the OMP algorithm on an input dictionary $B$ and a measurement $y$.

\subsection{Conditions on the Data-Generating Process}\label{setting}
For the data-generating process $y \sim Az + \epsilon$, it is in general impossible to successfully perform the decoding step in Algorithm \ref{genalgo} even with access to the ground-truth dictionary $A$. As a result, several conditions have been identified in the literature under which it is possible to provide guarantees on the success of decoding the sparse representation $z$. We recall two of the most common ones~\citep{rip}.

\begin{restatable}{definition}{rip}[{\it Restricted Isometry Property (RIP)}]\label{ripdef}
We say that a matrix $A \in \R^{d \times p}$ satisfies $(s, \delta_s)$-RIP if the following holds for all $s$-sparse $x \in \R^p$:
\begin{align*}
(1 - \delta_s) \norm{x}^2 \leq \norm{Ax}^2 \leq (1 + \delta_s) \norm{x}^2 \numberthis \label{ripeq}
\end{align*}
\end{restatable}

\begin{restatable}{definition}{incoherent}[{\it $\mu$-Incoherence}]\label{incoherent}
A matrix $A \in R^{d \times p}$ with unit norm columns is $\mu$-incoherent if:
\begin{align*}
\abs{\inrprod{A_i}{A_j}} \leq \mu \quad \text{for all } i \neq j \numberthis \label{incoeq}
\end{align*}
\end{restatable}

These two properties are closely related. For example, as a consequence of the Gershgorin circle theorem, $(\delta_s/s)$-incoherent matrices must satisfy $(s, \delta_s)$-RIP. 

Given the prominence of RIP and incoherence conditions in the compressed sensing and sparse coding literature, there has been a large body of work investigating families of matrices that satisfy these conditions. We refer the reader to \citet{randomrip} for an elegant proof that a wide class of random matrices in $\R^{d \times p}$ (i.e. subgaussian) satisfy $(k, \delta)$-RIP with high probability depending on $\delta$, $k$, $p$, and $d$. For an overview of deterministic constructions of such matrices, we refer the reader to \citet{detrip} and the references therein.

\section{Main Results}\label{main}
Having established the necessary background, we now present our main results. 
Our first result shows that minimizing the population reconstruction loss with a hard-sparsity constraint can lead to learning a dictionary $B$ that is far from the ground truth. We specifically work with the loss defined as:
\begin{align*}
L(B, k) = \E_y \left[\min_{\norm{\hat{z}}_0\leq k} \norm{y - B \hat{z}} ^ 2\right] \numberthis \label{recon2}
\end{align*}
Note that in the definition of $L(B, k)$, we are considering an NP-hard optimization problem (exhaustively searching over all $k$-sparse supports). We could instead replace this exhaustive optimization with an alternative least-squares-based approach (so long as it is at least as good as performing least squares on a uniformly random choice of $k$-sparse support), and our proof techniques for Theorem \ref{overfit} would still work. We consider this version only to simplify the presentation. 

We now show that, under appropriate settings, there exists a dictionary $B$ whose population loss $L(B, k)$ is smaller than that of $A$, while $d_R(A, B)$ is bounded away from $0$ by a term related to the noise in the data-generating process.

\begin{restatable}{assumption}{overfitsetting}\label{as1}
    Let $A \in \R^{d \times p}$ be an arbitrary matrix with unit-norm columns satisfying $(2k, \delta)$-RIP for $k=o\left(\frac{d}{\log p}\right)$ and $\delta = o(1)$, and suppose $\sigma_{\min}^2(A) = \Omega(p/d)$.
    We assume each measurement $y$ is generated as $y \sim Az + \epsilon$, where $z$ is a random vector drawn from an arbitrary probability measure $\Prob_z$ on $k$-sparse vectors in $\R^p$, and $\epsilon \sim \Gauss(0, \sigma^2 \Id_d)$ for some $\sigma > 0$.
\end{restatable}
\begin{restatable}{theorem}{overfit}[Overfitting to Reconstruction Loss]\label{overfit}
Consider the data-generating model in Assumption \ref{as1} and define $\Lambda(z)$ to be:
\begin{align*}
    \Lambda(z) = \inf \{t \ | \ \Prob_z(\norm{z} \geq t) \leq 1/d\}. \numberthis \label{lamz}
\end{align*}
Then for $q = \Omega(p^2 \max(d\sigma^2, \Lambda(z)^2)/\sigma^2)$, there exists a $B \in R^{d \times q}$ such that $L(B, k) \leq L(A, k) - \Omega(k\sigma^2)$ and $d_R(A, B) = \Omega(\sigma^2)$.
\end{restatable}
\textbf{Proof Sketch.} The key idea is to first determine how much the loss can be decreased by expanding from $k$-sparse combinations of the columns of $A$ to $2k$-sparse combinations, i.e., lower bound the gap between $L(A, k)$ and $L(A, 2k)$. After this, we can construct a dictionary $B$ whose columns form an $\epsilon$-net (with $\epsilon = \sigma^2$) for all $2$-sparse combinations of columns of $A$. Any $2k$-sparse combination of columns in $A$ can then be approximated as a $k$-sparse combination of columns in $B$, which is sufficient for proving the theorem.

\begin{remark}
Before we discuss the implications of Theorem \ref{overfit}, we first verify that Assumption~\ref{as1} is not vacuous, and in fact applies to many matrices of interest. This follows from a result of \citet{rudelson}, which shows that after appropriate rescaling, rectangular matrices with i.i.d. subgaussian entries satisfy the singular value condition in Assumption \ref{as1}. Furthermore, such matrices will also satisfy the RIP condition so long as $k$ is not too large relative to $d$ and $p$, as per \citet{randomrip} as discussed in the previous section.
\end{remark}

Theorem \ref{overfit} shows that learning an appropriately over-realized dictionary fails to recover the ground truth \textit{independent} of the distribution of $z$. 
This means that even if we let the norm of the signal $Az$ in the data-generating process be arbitrarily large, with sufficient over-realization we may still fail to recover the ground-truth dictionary $A$ by minimizing $L(B, k)$.

We also observe that the amount of over-realization necessary in Theorem \ref{overfit} depends on how well $z \sim \Prob_z$ can be bounded with reasonably high probability. If $z$ is almost surely bounded (as is frequently assumed), we can obtain the following cleaner corollary of Theorem \ref{overfit}.

\begin{restatable}{corollary}{overfit2}\label{overfit2}
Consider the same settings as Theorem \ref{overfit} with the added stipulation that $\Prob_z(\norm{z} \leq C) = 1$ for a universal constant $C$. Then for $q = \Omega(p^2 d)$, there exists a $B \in R^{d \times q}$ such that $d_R(A, B) = \Omega(\sigma^2)$ and $L(B, k) \leq L(A, k) - \Omega(k\sigma^2)$.
\end{restatable}

The reason that we can obtain a smaller population loss than the ground truth in Theorem \ref{overfit} is because we can make use of the extra capacity in $B$ to overfit the noise $\epsilon$ in the data-generating process. To prevent this, our key idea is to perform the decoding step $B\hat{z} = y$ on a subset of the dimensions of $y$ - which we refer to as the \textit{unmasked} part of $y$ - and then evaluate the loss of $B$ using the complement of that subset (the \textit{masked} part of $y$). Intuitively, because each coordinate of the noise $\epsilon$ is independent, a dictionary $B$ that well-approximates the noise in the unmasked part of $y$ will have no benefit in approximating the noise in the masked part of $y$.

We can formalize this as the following masking objective:
\begin{align*}
    L_{mask}(B, k, M) &= \E_y \left[\norm{[y]_{[d] \setminus M} - [B\hat{z}]_{[d] \setminus M}}^2\right] \numberthis \label{LMloss} \\
    \text{where } \hat{z}([y]_M) &= g_{\textnormal{OMP}}([y]_M, P_M B, k) \numberthis \label{LMlossinner}
\end{align*}

In defining $L_{mask}$, we have opted to use $g_{\textnormal{OMP}}$ in the inner minimization step, as opposed to the exhaustive $\argmin$ in the definition of $L$. Similar to the discussion earlier, we could have instead used any other approach based on least squares to decode $\hat z$ (including the exhaustive approach), so long as we have guarantees on the probability of failing to recover the true code $z$ given access to the ground-truth dictionary $A$. This choice of using OMP is mostly to tie our theory more closely with our experiments.

Now we present our second main result which shows, in contrast to Theorem \ref{overfit}, that optimizing $L_{mask}$ prevents overfitting noise (albeit in a different but closely related setting).
\begin{restatable}{assumption}{masksetting}\label{as2}
    Let $A \in \R^{d \times p}$ be an arbitrary matrix such that there exists an $M \subset [d]$ with $P_M A$ being $\mu$-incoherent with $\mu \leq C/(2k - 1)$ for a universal constant $C < 1$.
    We assume each measurement $y$ is generated as $y \sim Az + \epsilon$, where $[z]_{\supp(z)} \sim \Gauss(0, \sigma_z^2 \Id_k)$ with $\supp(z)$ drawn from an arbitrary probability distribution over all size-$k$ subsets of $[d]$, and $\epsilon \sim \Gauss(0, \sigma^2 \Id_d)$ for some $\sigma > 0$.
\end{restatable}
\begin{restatable}{theorem}{masking}[Benefits of Masking]\label{masking}
    Consider the data-generating model in Assumption \ref{as2}.
    For any non-empty mask $M \subset [d]$ such that $P_M A$ satisfies the $\mu$-incoherence condition in the assumption, we have
    \begin{align*}
        \lim_{\sigma_z \to \infty} \left(L_{mask}(A, k, M) - \min_{B} L_{mask}(B, k, M)\right)  = 0 \numberthis \label{Abest}
    \end{align*}
    That is, as the expected norm of the signal $Az$ increases, there exist minimizers $B$ of $L_{mask}$ such that $d_R(A, B) \to 0$.
\end{restatable}

\textbf{Proof Sketch.} The proof proceeds by expanding out $L_{mask}(B, k, M)$ and using the fact that $[B\hat{z}^*]_{[d] \setminus M}$ is independent of $[\epsilon]_{[d] \setminus M}$ to obtain a quantity that closely resembles the prediction risk considered in analyses of linear regression. From there we show that the Bayes risk is lower bounded by the risk of a regularized least squares solution with access to a support oracle. We then rely on a result of \citet{omperror} to show that $g_{\textnormal{OMP}}([y]_M)$ recovers the support of $z$ with increasing probability as $\sigma_z \to \infty$, and hence its risk converges to the aforementioned prediction risk.

\begin{remark}
As before, so long as the mask $M$ is not too small (i.e. $\Omega(k \polylog(p))$), matrices with i.i.d. subgaussian entries will satisfy the assumptions on $A$ in Assumption \ref{as2}. In particular, the set of ground truth dictionaries satisfying Assumptions \ref{as1} \textit{and} \ref{as2} is non-trivial, once again by results in \citet{rudelson}.
\end{remark}

Comparing Theorem \ref{masking} to Theorem \ref{overfit}, we see that approximate minimizers of $L_{mask}$ can achieve arbitrarily small recovery error, so long as the signal $Az$ is large enough; whereas for $L$, there always exist minimizers whose recovery error is bounded away from $0$. We note that having the expected norm of the signal be large is effectively necessary to hope for recovering the ground truth in our setting, as in the presence of Gaussian noise there is always some non-zero probability that the decoding step can fail. Full proofs of Theorems \ref{overfit} and \ref{masking} can be found in Section \ref{fullproofs} of the Appendix.

\section{Experiments}\label{experiments}
In this section, we examine whether the separation between the performance of sparse coding with or without masking (demonstrated by Theorems \ref{overfit} and \ref{masking}) manifests in practice. Code for the experiments in this section can be found at: \url{https://github.com/2014mchidamb/masked-sparse-coding-icml}.

For our experiments, we need to make a few concessions from the theoretical settings introduced in Sections \ref{sparsebgrnd} and \ref{setting}. Firstly, we cannot directly optimize the expectations in $L$ and $L_{mask}$ as defined in Equations \eqref{recon2} and \eqref{LMloss}, so we instead optimize the corresponding empirical versions defined in the same vein as Equation \eqref{emprecon1}. Another issue is that the standard objective $L$ requires solving the optimization problem $\min_{\norm{\hat{z}}_0 \leq k} \norm{y - B\hat{z}}^2$, which is NP-hard in general. In order to experiment with reasonably large values of $d, p, \text{ and } p'$ and to be consistent with the decoding step in $L_{mask}$, we thus approximately solve the aforementioned optimization problem using OMP. 

\begin{algorithm}[ht]
   \caption{Algorithm for Optimizing $L$}
\begin{algorithmic}\label{lalgo}
   \STATE {\bfseries Input:} Data $\{y_1, ..., y_T\}$, Dictionary $B^{(0)} \in \R^{d \times p'}$, Learning Rate $\eta \in \R^+$
   \FOR{$t=0$ {\bfseries to} $T-1$}
   \STATE $z \gets g_{\textnormal{OMP}}(y_{t + 1}, B^{(t)})$
   \STATE $B^{(t + 1)} \gets B^{(t)} - \eta \grad_{B^{(t)}} \norm{y_{t + 1} - B^{(t)}z}^2$
   \STATE $B^{(t + 1)} \gets \mathrm{Proj}_{\mathbb{S}^{d - 1}} B^{(t + 1)}$
   \ENDFOR
\end{algorithmic}
\end{algorithm}

\begin{algorithm}[ht]
   \caption{Algorithm for Optimizing $L_{mask}$}
\begin{algorithmic}\label{maskalgo}
   \STATE {\bfseries Input:} Data $\{y_1, ..., y_T\}$, Dictionary $B^{(0)} \in \R^{d \times p'}$, Learning Rate $\eta \in \R^+$, Mask Size $m \in [d]$
   \FOR{$t=0$ {\bfseries to} $T-1$}
   \STATE $M \gets \text{Uniformly random subset of size $m$ from } [d]$
   \STATE $z \gets g_{\textnormal{OMP}}([y_{t + 1}]_M, B^{(t)})$
   \STATE $B^{(t + 1)} \gets B^{(t)} - \eta \grad_{B^{(t)}} \norm{[y_{t + 1}]_{M^c} - [B^{(t)} z]_{M^c}}^2$
   \STATE $B^{(t + 1)} \gets \mathrm{Proj}_{\mathbb{S}^{d - 1}} B^{(t + 1)}$
   \ENDFOR
\end{algorithmic}
\end{algorithm}

The approaches for optimizing $L$ and $L_{mask}$ given $n$ samples from the data-generating process are laid out in Algorithms \ref{lalgo} and \ref{maskalgo}, in which we use $\mathrm{Proj}_{\mathbb{S}^{d - 1}} B$ to denote the result of normalizing all of the columns of $B$. 
We also use $M^c$ as a shorthand in Algorithm \ref{maskalgo} to denote $[d] \setminus M$.
 
We point out that Algorithm \ref{maskalgo} introduces some features that were not present in the theory of the masking objective; namely, in each iteration, we randomly sample a new mask of a pre-fixed size. 
This is because if we were to run gradient descent using a single, fixed mask $M$ at each iteration, as we don't differentiate through the OMP steps, the gradient with respect to $B^{(t)}$ computed on the error $\norm{[y]_{[d] \setminus M} - [Bz]_{[d] \setminus M}}^2$ would be non-zero only for those rows of $B$ corresponding to the indices $[d] \setminus M$. To avoid this issue, we sample new masks in each iteration so that each entry of $B$ can be updated. There are alternative approaches that can achieve similar results; i.e. deterministically cycling through different masks, but they have similar performance. 

\begin{figure*}[htp]
\centering
\subfigure[Sample init, $p'$ scaling]{\includegraphics[scale=0.24]{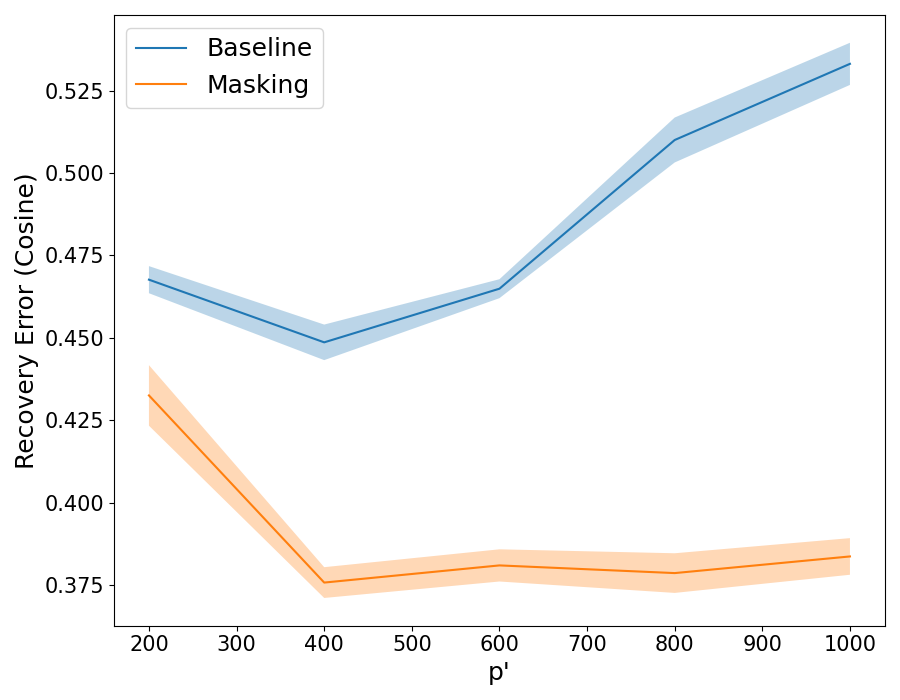}}
\subfigure[Sample init, $d, p, k, p'$ scaling]{\includegraphics[scale=0.24]{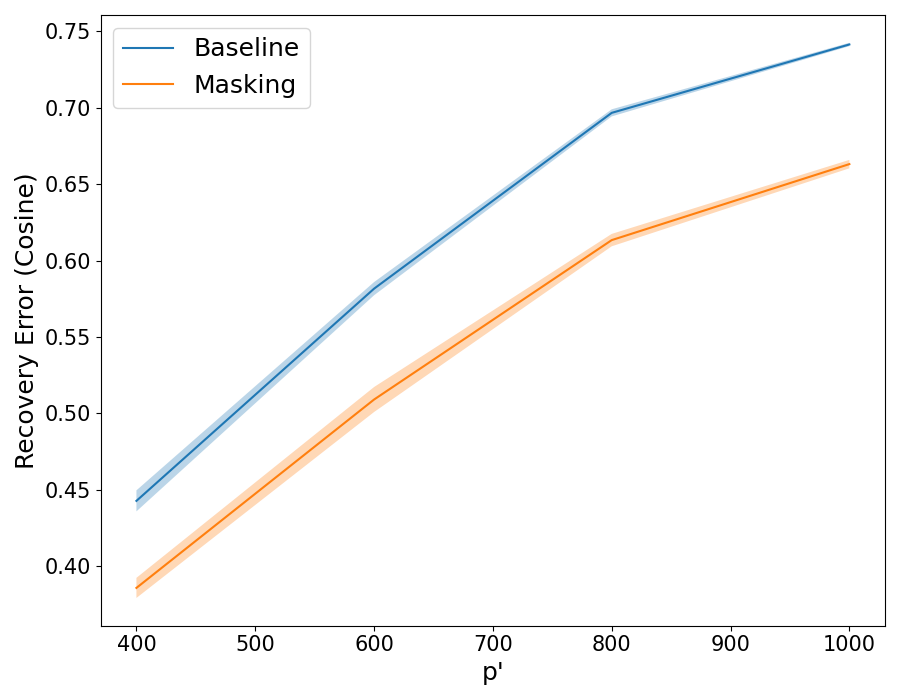}}
\subfigure[Sample init, noise scaling]{\includegraphics[scale=0.24]{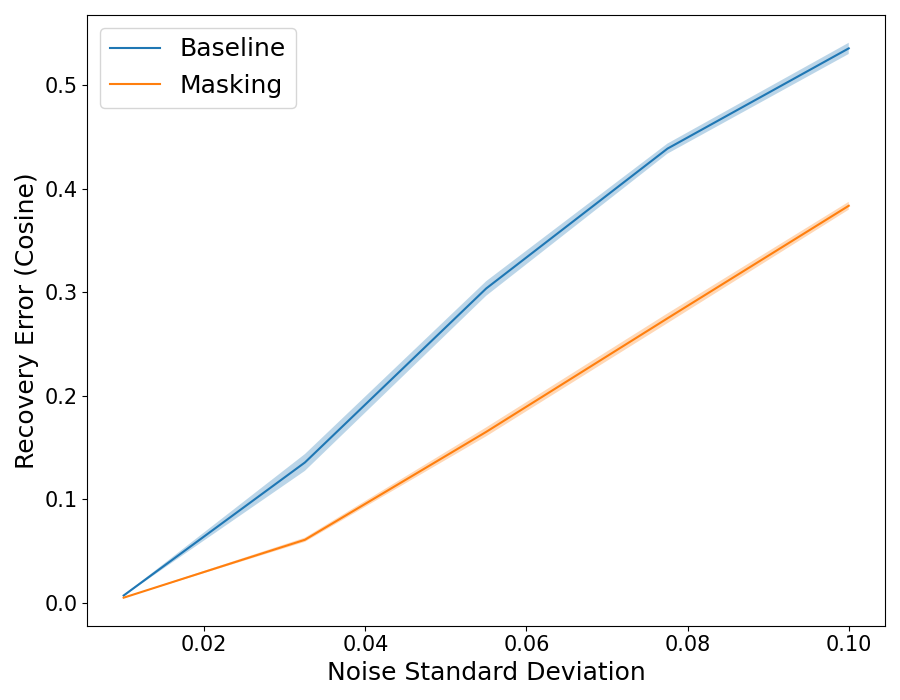}}
\subfigure[Local init, $p'$ scaling]{\includegraphics[scale=0.24]{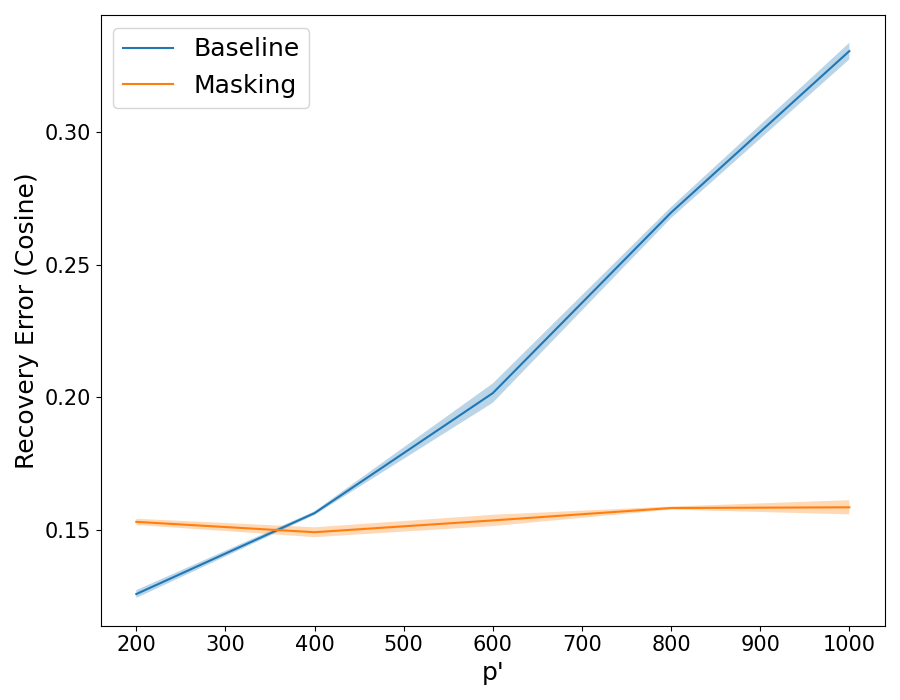}}
\subfigure[Local init, $d, p, k, p'$ scaling]{\includegraphics[scale=0.24]{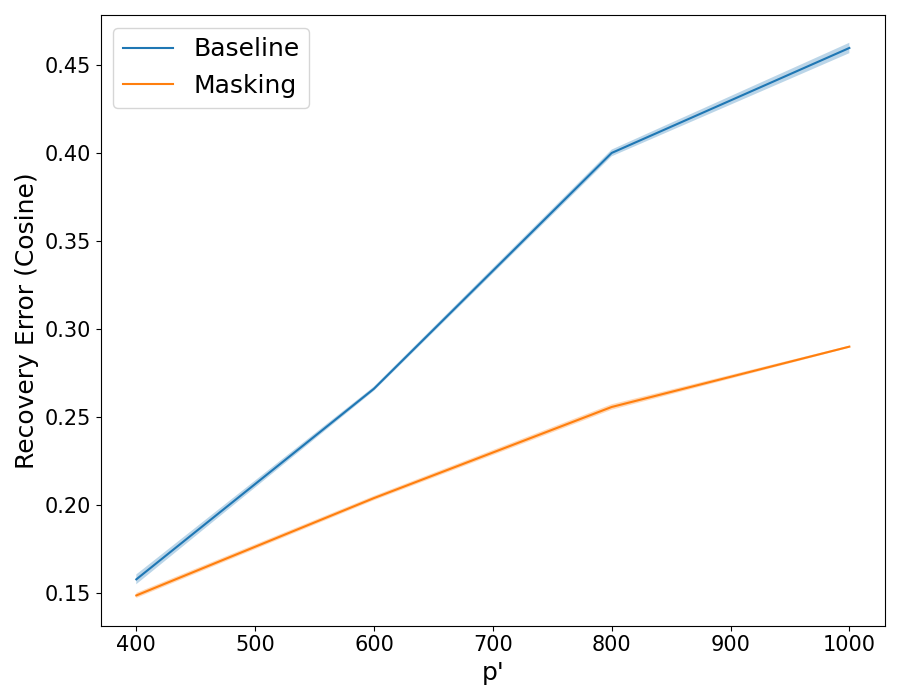}}
\subfigure[Local init, noise scaling]{\includegraphics[scale=0.24]{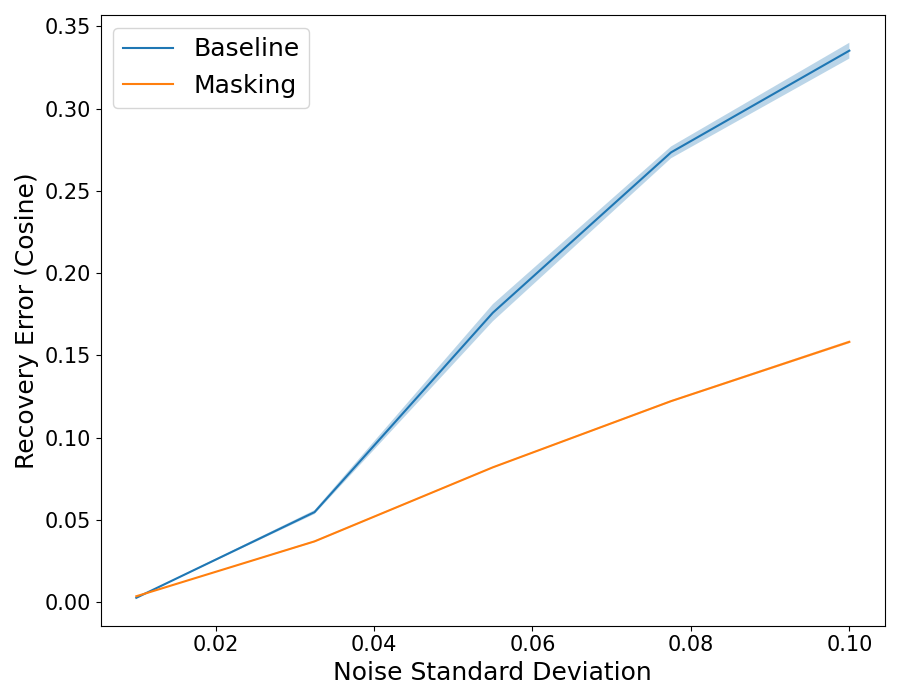}}
\caption{Comparison of Algorithm \ref{lalgo} (Baseline) and Algorithm \ref{maskalgo} (Masking) under the settings of Sections \ref{scaleover}, \ref{scaleall}, and \ref{scalenoise} (from left to right). Each curve represents the mean of 5 training runs, with the surrounded shaded area representing one standard deviation.}
\label{scaleoverfig}
\end{figure*}

While we will analyze the performance of Algorithms \ref{lalgo} and \ref{maskalgo} across several different experimental setups over the next few subsections, we describe the following facets shared across all setups. We generate a dataset of $n = 1000$ samples $y_i = Az_i + \epsilon_i$, where $A \in \R^{d \times p}$ is a standard Gaussian ensemble with normalized columns, the $z_i$ have uniformly random $k$-sparse supports whose entries are i.i.d. $\Gauss(0, 1)$, and the $\epsilon_i$ are mean zero Gaussian noise with some fixed variance (which we will vary in our experiments). We also normalize the $z_i$ so as to constrain ourselves to the bounded-norm setting of Corollary \ref{overfit2}. In addition to the $1000$ samples constituting the dataset, we also assume access to a held-out set of $p'$ samples from the data-generating process for initializing the dictionary $B^{(0)} \in \R^{d \times p'}$.

For training, we use batch versions of Algorithms \ref{lalgo} and \ref{maskalgo} in which we perform gradient updates with respect to the mean losses computed over $\{y_1, ..., y_B\}$ with $B = 200$ as the batch size. For the actual gradient step, we use Adam \citep{adam} with its default hyperparameters of $\beta_1 = 0.9, \beta_2 = 0.999$ and a learning rate of $\eta = 0.001$, as we found Adam trains \textit{significantly} faster than SGD (and we ran into problems when using large learning rates for SGD). We train for $500$ epochs (passes over the entire dataset) for both Algorithms \ref{lalgo} and \ref{maskalgo}. For Algorithm \ref{maskalgo}, we always use a mask size of $d - \lfloor d/10 \rfloor$, which we selected based off of early experiments. 
We ensured that, even for this fairly large mask size, the gradient norms for both $L$ and $L_{mask}$ were of the same order in our experiments and that 500 epochs were sufficient for training.

We did not perform extensive hyperparameter tuning, but we found that the aforementioned settings performed better than the alternative choices we tested for both algorithms across all experimental setups. 
Our implementation is in PyTorch \citep{pytorch}, and all of our experiments were conducted on a single P100 GPU.

\subsection{Scaling Over-realization}\label{scaleover}
We first explore how the choice of $p'$ for the learned dictionary $B$ affects the empirical performance of Algorithms \ref{lalgo} and \ref{masking} when the other parameters of the problem remain fixed. Theorem \ref{overfit} and Corollary \ref{overfit2} indicate that the performance of Algorithm \ref{lalgo} should suffer as we scale $p'$ relative to $d$ and $p$.

In order to test whether this is actually the case in practice, we consider samples generated as described above with $A \in \R^{d \times p}$ for $d = 100, p = 200, \text{ and } \norm{z}_0 = k = 5$ fixed, while scaling the number of atoms $p'$ in $B$ from $p' = p$ (exactly realized) to $p' = n$ (over-realized and overparameterized). We choose $\epsilon_i \sim \Gauss(0, 1/d)$, which is a high noise regime as the expected norm of the noise $\epsilon_i$ will be comparable to that of the signal $Az_i$. To make it computationally feasible to run several trials of our experiments, we consider the $p'$ values $\{200, 400, 600, 800, 1000\}$ and do not consider more fine-grained interpolation between $p$ and $n$.

For the training process, we consider two different initialization of $B^{(0)}$. In the first case, we initialize $B^{(0)}$ to have columns corresponding to the aforementioned set of held-out $p'$ (normalized) samples from the data-generating process, as this is a standard initialization choice that has been known to work well in practice \citep{arora15}. However, this initialization choice in effect corresponds to a dataset of $n + p'$ samples, and it is fair to ask whether this initialization benefit is worth the sample cost relative to a random initialization. Our initial experiments showed that this was indeed the case, i.e. random initialization with access to $p'$ additional samples did not perform better, so we focus on this sample-based initialization. That being said, we did not find the ordering of the performances of Algorithms \ref{lalgo} and \ref{maskalgo} sensitive to the initializations we considered, only the final absolute performance in terms of $d_R(A, B)$.

In addition to this purely sample-based initialization, we also consider a ``local'' initialization of $B^{(0)}$ to the ground truth $A$ itself concatenated with $p' - p$ normalized samples. This is obviously not intended to be a practical initialization; the goal here is rather to analyze the extent of overfitting to the noise $\epsilon_i$ in the dataset for both algorithms. Namely, we expect that Algorithm \ref{lalgo} will move further away from the ground truth than Algorithm \ref{maskalgo}.

The results for training using these initializations for both algorithms and then computing the final dictionary recovery errors $d_R(A, B)$ are shown in Figure \ref{scaleoverfig}(a, d). We use cosine distance when reporting the error $d_R(A, B)$ since the learned dictionary $B$ also has normalized columns, so Euclidean distance only changes the scale of the error curves and not their shapes.

For both choices of initialization, we observe that Algorithm \ref{maskalgo} outperforms Algorithm \ref{lalgo} as $p'$ increases, with this gap only becoming more prominent for larger $p'$. Furthermore, we find that recovery error actually \textit{worsens} for Algorithm \ref{lalgo} for every choice of $p' > p$ for both initializations in our setting. While this is possibly unsurprising for initializing at the ground truth, it is surprising for the sample-based initialization which does not start at a low recovery error. On the other hand, training using Algorithm \ref{maskalgo} improves the recovery error from initialization when using sample-based initialization for every choice of $p'$ except $p' = n$, which again corresponds to the overparameterized regime in which it is theoretically possible to memorize every sample as an atom of $B$.

Additionally, we also see that the performance of Algorithm \ref{maskalgo} is much less sensitive to the level of over-realization in $B$. When training from local initialization, Algorithm \ref{maskalgo} retains a near-constant level of error/overfitting as we scale $p'$. Similarly, when training from sample initialization, performance does not degrade as we scale $p'$, and in fact improves initially with a modest level of over-realization.

This improvement up to a certain amount of over-realization (in our case $p' = 2p$) is seen even in the performance of Algorithm \ref{lalgo} for sample initialization (although note that while the recovery error is better for $p' = 2p$ compared to $p' = p$, training still makes the error worse than initialization for Algorithm \ref{lalgo}). A similar phenomenon was observed in \citet{overrealized} in the setting where $y_i = Az_i$ (no noise), and we find it interesting that the phenomenon is (seemingly) preserved even in the presence of noise. We do not investigate the optimal level of over-realization any further, but believe it would be a fruitful direction for future work.

\subsection{Scaling All Parameters}\label{scaleall}
The experiments of Section \ref{scaleover} illustrate that for the fixed choices of $d$, $p$, and $k$ that we used, scaling the over-realization of $B$ leads to rapid overfitting in the case of Algorithm \ref{lalgo}, while Algorithm \ref{maskalgo} maintains good performance. To verify that this is not an artifact of the choices of $d, p, \text{ and } k$ that we made, we also explore what happens when over-realization is kept at a fixed ratio to the other setting parameters while they are scaled.

For these experiments, we consider $A \in \R^{d \times p}$ for $d \in \{100, 150, 200, 250\}$ and scale $p$ as $p = 2d$ and $k$ as $k = \lfloor d/20 \rfloor$ to (approximately) preserve the ratio of atoms and sparsity to dimension from the previous subsection. We choose to scale $p'$ as $p' = 2p$ since that was the best-performing setting (for the baseline) from the experiments of Figure \ref{scaleoverfig}. We keep the noise variance at $1/d$ to stay in the relatively high noise regime. 

As before, we consider a sample-based initialization as well as a local initialization near the ground truth dictionary $A$. The results for both Algorithms \ref{lalgo} and \ref{maskalgo} under the described parameter scaling are shown in Figure \ref{scaleoverfig}(b, e). Once again we find that Algorithm \ref{maskalgo} has superior recovery error, with this gap mostly widening as the parameters are scaled. However, unlike the case of fixed $d, p, \text{ and } k$, this time the performance of Algorithm \ref{maskalgo} also degrades with the scaling. This is to be expected, as increasing $p$ leads to more ground truth atoms that need to be recovered well in order to have small $d_R(A, B)$.

\subsection{Analyzing Different Noise Levels}\label{scalenoise}
The performance gaps shown in the plots of Figure \ref{scaleoverfig}(a, b, d, e) are in the high noise regime, and thus it is fair to ask whether (and to what extent) these gaps are preserved at lower noise settings. We thus revisit the settings of Section \ref{scaleover} (choosing $d, p, \text{ and } k$ to be the same) and fix $p' = 1000$ (the maximum over-realization we consider). We then vary the variance of the noise $\epsilon_i$ from $1/d^2$ to $1/d$ linearly, which corresponds to the standard deviations of the noise being $\{0.01, 0.0325, 0.055, 0.0775, 0.1\}$.

Results are shown for the sample-based initialization as well as the local initialization in Figure \ref{scaleoverfig}(c, f). Here we see that when the noise variance is very low, there is virtually no difference in performance between Algorithms \ref{lalgo} and \ref{maskalgo}. Indeed, when the variance is $1/d^2$ we observe that both algorithms are able to near-perfectly recover the ground truth, even from the sample-based initialization.

However, as we scale the noise variance, the gap between the performance of the two algorithms resembles the behavior seen in the experiments of Sections \ref{scaleover} and \ref{scaleall}.

\section{Conclusion}
In summary, we have shown in Sections \ref{main} and \ref{experiments} that applying the standard frameworks for sparse coding to the case of learning over-realized dictionaries can lead to overfitting the noise in the data. In contrast, we have also shown that by carefully separating the data used for the decoding and update steps in Algorithm \ref{genalgo} via masking, it is possible to alleviate this overfitting problem both theoretically and practically. Furthermore, the experiments of Section \ref{scalenoise} demonstrate that these improvements obtained from masking are not at the cost of worse performance in the low noise regime, indicating that a practitioner may possibly use Algorithm \ref{maskalgo} as a drop-in replacement for Algorithm \ref{lalgo} when doing sparse coding.

Our results also raise several questions for exploration in future work. Firstly, in both Theorem \ref{masking} and our experiments we have constrained ourselves to the case of sparse signals that follow Gaussian distributions. It is natural to ask to what extent this is necessary, and whether our results can be extended (both theoretically and empirically) to more general settings (we expect, at the very least, that parts of Assumptions \ref{as1} and \ref{as2} can be relaxed). Additionally, we have focused on sparse coding under hard-sparsity constraints and using orthogonal matching pursuit, and it would be interesting to study whether our ideas can be used in other sparse coding settings.

Beyond these immediate considerations, however, the intent of our work has been to show that there is still likely much to be gained from applying ideas from recent developments in areas such as self-supervised learning to problems of a more classical nature such as sparse coding. Our work has only touched on the use of a single such idea (masking), and we hope that future work looks into how other recently popular ideas can potentially improve older algorithms.  

Finally, we note that this work has been mostly theoretical in nature, and as such do not anticipate any direct misuses or negative impacts of the results.

\section*{Acknowledgements}
Rong Ge, Muthu Chidambaram, and Chenwei Wu are supported by NSF Award DMS-2031849, CCF-1845171 (CAREER), CCF-1934964 (Tripods), and a Sloan Research Fellowship. Yu Cheng is supported in part by NSF Award CCF-2307106.

\bibliography{icml2023}
\bibliographystyle{icml2023}

\newpage
\appendix
\onecolumn
\section{Full Proofs}\label{fullproofs}

\subsection{Proof of Theorem \ref{overfit}}
We first recall the setting of Theorem \ref{overfit}.
\overfitsetting*

And now we prove:
\overfit*
\begin{proof}
Our proof technique will be to first lower bound the gap $L(A, k) - L(A, 2k)$, and then to construct a $B$ matrix that closely approximates the $2k$-sparse combinations of the columns of $A$.

From the definition of $L(B, k)$ we have that:
\begin{align*}
    L(A, k) - L(A, 2k) = \E_y \left[\min_{\norm{\hat{z}}_0 = k} \norm{y - A \hat{z}} ^ 2\right] - \E_y \left[\min_{\norm{\hat{z}}_0 = 2k} \norm{y - A \hat{z}} ^ 2\right] \numberthis \label{LAdiff}
\end{align*}
Now let $\hat{z}^*(y) = \argmin_{\norm{\hat{z}}_0 = k} \norm{y - A \hat{z}}$ and $S^* = \supp{(\hat{z}^*)}$, and further define:
\begin{align*}
    \Tilde{z}(y) = \argmin_{\norm{\hat{z}}_0 = k} \norm{(y - A\hat{z}^*(y)) - A \hat{z}} \numberthis \label{tildez}
\end{align*}
We will also use $\Tilde{S} = \supp{(\Tilde{z}(y))}$. For convenience, we will write $\hat{z}^*$ and $\tilde{z}$ when $y$ is clear from context. Applying this notation to Equation \eqref{LAdiff} gives:
\begin{align*}
    L(A, k) - L(A, 2k) &\geq \E_y [\norm{y - A \hat{z}^*}^2] - \E_y [\norm{y - A \hat{z}^* - A \Tilde{z}}^2] \\
    &= \E_y [2\inrprod{y - A \hat{z}^*}{A \tilde{z}}] - \E_y [\norm{A \tilde{z}}^2] \\
    &= \E_y [\norm{A \tilde{z}}^2] \numberthis \label{LAdifflb}
\end{align*}

Where we obtained the last line above by using the fact that $A \tilde{z}$ is the orthogonal projection of $y - A \hat{z}^*$ on to the span of $A_{\Tilde{S}}$. Now using the fact that $A$ is $(2k, \delta)$-RIP we have that:
\begin{align*}
    \E_y [\norm{A \tilde{z}}^2] &\geq (1 - \delta)^2 \E_y \left[\norm{\tilde{z}}^2\right] \\
    &\geq (1 - \delta)^2 \E_y \left[\norm{A_{\Tilde{S}}^{+} (y - A\hat{z}^*)}^2\right] \\
    &\geq (1 - \delta)^4 \E_y \left[\norm{A_{\Tilde{S}}^{T} (y - A\hat{z}^*)}^2\right] \\
    &= (1 - o(1)) \E_y \left[\norm{A_{\Tilde{S}}^{T} (y - A\hat{z}^*)}^2\right] \numberthis \label{ripstep}
\end{align*}

Where above we used the fact that $A_{\Tilde{S}}^{+} = (A_{\Tilde{S}}^T A_{\Tilde{S}})^{-1} A_{\Tilde{S}}^{T}$ and RIP to obtain $\norm{(A_{\Tilde{S}}^T A_{\Tilde{S}})^{-1}}_{op} \geq 1/(1 + \delta)$, which led to the penultimate step. It remains to compute (or lower bound) the expectation in Equation \eqref{ripstep}. Towards this end, we let $S$ denote a (uniformly) random subset of size $k$ from $[p]$. Then we have that (using Assumption \ref{as1}):
\begin{align*}
    \E_y \left[\norm{A_{\Tilde{S}}^{T} (y - A\hat{z}^*)}^2\right] &\geq \E_y \left[\E_S \left[\norm{A_{S}^{T} (y - A\hat{z}^*)}^2\right]\right] \\
    &= \frac{k}{p} \E_y \left[\norm{A^{T} (y - A\hat{z}^*)}^2\right] \\
    &\geq \frac{k}{p} \sigma_{\min}^2(A) \E_y \left[\norm{y - A\hat{z}^*}^2\right] \\
    &= \Omega\left( \frac{k}{d} \E_y \left[\norm{y - A\hat{z}^*}^2\right] \right) \numberthis \label{randomS}
\end{align*}

Now the expectation in Equation \eqref{randomS} can be lower bounded in the same vein as Equation \eqref{ripstep} (i.e. relying on the RIP property). Below we use $S_{2k}^*$ to denote the optimal support for the minimization problem.
\begin{align*}
    \E_y \left[\norm{y - A\hat{z}^*}^2\right] &= \E_{z, \epsilon} \left[\min_{\norm{\hat{z}}_0 = k} \norm{\epsilon - A(\hat{z} - z)}^2\right] \\
    &\geq \E_{\epsilon} \left[\min_{\norm{\hat{z}}_0 = 2k} \norm{\epsilon - A\hat{z}}^2\right] \\
    &= \E_{\epsilon} \left[\norm{\epsilon}^2\right] - 2\E_{\epsilon}\left[\inrprod{A_{S_{2k}^*}^+ \epsilon}{\epsilon}\right] + \E_{\epsilon}\left[\norm{A_{S_{2k}^*}^+ \epsilon}^2\right] \\
    &\geq \E_{\epsilon} \left[\norm{\epsilon}^2\right] - (1 + o(1)) \E_{\epsilon} \left[\norm{A_{S_{2k}^*}^T \epsilon}^2\right] \\
    &\geq d\sigma^2 - (1 + o(1)) 2k\E_{\epsilon}\left[\max_{i \in [p]} (A_i^T \epsilon)^2\right] \\
    &\geq d\sigma^2 - O\left(k\sigma^2 \log p\right) \\
    &= \Omega(d\sigma^2) \numberthis \label{finalexp}
\end{align*}

Where above we used Lemma \ref{max-chi-square} since the random variables $(A_i^T \epsilon)^2$ follows a scaled chi-square distribution with degree of freedom 1, and for the last line we use $k=o\left(\frac{d}{\log p}\right)$. Now putting Equations \eqref{LAdifflb}-\eqref{finalexp} together, we obtain:
\begin{align*}
    L(A, k) - L(A, 2k) \geq \Omega(k\sigma^2) \numberthis \label{improvement}
\end{align*}

Given the gap between $L(A, k)$ and $L(A, 2k)$ shown in Equation \eqref{improvement}, our goal is now to construct a matrix $B$ such that we can approximate sufficiently large $2k$-sparse combinations of the columns of $A$ via $B\hat{z}$ (where $\hat{z}$ is $k$-sparse). We recall from standard concentration of measure arguments (see \citet{hdp} for details) that $\Prob(\norm{\epsilon}^2 \geq 2d\sigma^2) \leq \exp(-\Omega(d))$. Furthermore, by Assumption \ref{as1}, $\norm{Az} \leq \Lambda(z) (1 + o(1))$ with probability at least $1 - 1/d$. Thus, we only need the columns of $B$ to approximate $Ax$ for 2-sparse $x$ (since we are interested in $B\hat{z}$ and $\hat{z}$ is $k$-sparse) and $\norm{Ax} \leq \gamma_1\max(\sigma\sqrt{d}, \Lambda(z))$ for an appropriately large constant $\gamma_1$ (as this will imply we get the same gap as Equation \ref{finalexp}).

To do this, we can construct $\epsilon$-nets for each of the following sets (indexed by the different possible 2-sparse supports $S \subset [p]$):
\begin{align*}
    V_S = \{Ax \ | \ \supp{(x)} = S, \ \norm{Ax} \leq \gamma_1\max(\sigma\sqrt{d}, \Lambda(z))\} \numberthis \label{enetsets}
\end{align*}

Since $A$ has $p$ columns, we need $\Theta(p^2)$ such $\epsilon$-nets. As long as we choose $\epsilon = \gamma_2\sigma^2$ with $\gamma_2$ a constant, we can approximate $2k$-sparse combinations of the columns of $A$ with error $k\gamma_2\sigma^2$ using $k$-sparse combinations from these nets, which is sufficient for our purposes given the result of Equation \eqref{improvement}.

Now let the columns of $B$ be the union of the $\epsilon$-nets for the sets $V_S$ and define $\mathcal{E} = \{\norm{Az} + \norm{\epsilon} \leq \gamma_1\max(\sigma\sqrt{d}, \Lambda(z))\}$. After choosing $\gamma_2$ to be sufficiently small, we then get from Equations \eqref{LAdifflb}-\eqref{improvement} and the fact that $\Prob\left(\mathcal{E}\right) \geq 1 - 1/d$:
\begin{align*}
    L(A, k) - L(B, k) &\geq \E_y \left[\norm{A \tilde{z}}^2 \ \big\vert\ \mathcal{E} \right] \ \Prob\left(\mathcal{E}\right) - k\gamma_2\sigma^2 \\
    &= \Omega(k\sigma^2) \numberthis \label{finalresult}
\end{align*}

Noting that the $\epsilon$-nets for each $V_S$ are of size $O(\max(d\sigma^2, \Lambda(z)^2)/\sigma^2)$ from our choice of $\epsilon$ (once again, refer to \citet{hdp} for bounds on the size of $\epsilon$-nets), this construction of $B$ requires $O(p^2 \max(d\sigma^2, \Lambda(z)^2)/\sigma^2)$ columns. As we can choose these columns to be different from those of $A$ by $\gamma_2\sigma^2$ (in norm), we obtain the desired result.
\end{proof}

\subsection{Proof of Theorem \ref{masking}}
Again, we first recall the setting of Theorem \ref{masking}.
\masksetting*

In order to prove Theorem \ref{masking}, we will need a result from \citet{omperror}, which we restate below.
\begin{theorem}[Theorem 9 in \citet{omperror}]\label{omperrorthm}
For $y \sim Az + \epsilon$ with $\epsilon \sim \Gauss(0, \sigma^2 \Id)$ and $A \in \R^{d \times p}$ being $\mu$-incoherent with $\mu < 1/(2k - 1)$, let us define:
\begin{align*}
    \mathcal{S} = \left\{A_i : i \in [p],\ \abs{z_i} \geq \frac{2\sigma \sqrt{k} \sqrt{2(1 + \eta)\log p}}{1 - (2k - 1)\mu}\right\} \numberthis \label{suppset}
\end{align*}
Then OMP (as defined in Algorithm) selects a column from $\mathcal{S}$ at each step with probability at least $1 - \frac{1}{p^{\eta}\sqrt{2\log p}}$ for $\eta \geq 0$.
\end{theorem}

Now we may prove:
\masking*

\begin{proof}
We have from the definition of $L_{mask}$ that:
\begin{align*}
    L_{mask}(B, k, M) &= \E_{z, \epsilon} \left[\norm{[Az]_{[d] \setminus M} + [\epsilon]_{[d] \setminus M} - [B\hat{z}]_{[d] \setminus M}}^2\right] \\
    &=  \E_{z, \epsilon} \left[\norm{[Az]_{[d] \setminus M} - [B\hat{z}]_{[d] \setminus M}}^2\right] + E_{\epsilon} \left[\norm{[\epsilon]_{[d] \setminus M}}^2\right] - \E_{z, \epsilon} \left[\inrprod{[B\hat{z}]_{[d] \setminus M}}{[\epsilon]_{[d] \setminus M}}\right] \\
    &= \E_{z, \epsilon} \left[\norm{[Az]_{[d] \setminus M} - [B\hat{z}]_{[d] \setminus M}}^2\right] + (d - \abs{M})\sigma^2 \\
    &= \E_{z, \epsilon} \left[\norm{P_{[d] \setminus M}Az - P_{[d] \setminus M}B\hat{z}}^2\right] + (d - \abs{M})\sigma^2 \numberthis \label{redLMloss}
\end{align*}

Since $[B\hat{z}]_{[d] \setminus M}$ is necessarily independent\footnote{There is actually a slight technicality here; we need $g_{\textnormal{OMP}}$ to be Borel measurable, which is the case because it consists of the composition of Borel measurable functions.} of $[\epsilon]_{d \setminus M}$ (by the construction of $\hat{z}$). Now the quantity in Equation \ref{redLMloss} depending on $B$ looks almost identical to the prediction risk considered in linear regression.

With this in mind, let us define: 
\begin{align*}
    \mathcal{R}_M(\hat{y}) = \E_{z, \epsilon} \left[\norm{P_{[d] \setminus M}Az - \hat{y}}^2\right] \numberthis \label{predrisk}
\end{align*}
Where $\hat{y}$ is any estimator that depends only on $[y]_M$ (i.e. in the interest of brevity we are omitting writing $\hat{y}([Az + \epsilon]_M)$). We can lower bound Equation \eqref{redLMloss} by analyzing $\mathcal{R}_M(\hat{y})$:
\begin{align*}
    \inf_{\hat{y}} \mathcal{R}_M &= \inf_{\hat{y}} \E_{z, \epsilon} \left[\norm{P_{[d] \setminus M}Az - \hat{y}}^2\right] \\
    &= \inf_{\hat{y}} \E_{z, \epsilon} \left[ \E_{z, \epsilon} \left[\norm{P_{[d] \setminus M}A_{\supp(z)}[z]_{\supp(z)} - \hat{y}}^2 \ \bigg\vert\ \supp(z) \right]  \right] \numberthis \label{bayesrisk}
\end{align*}
Equation \eqref{bayesrisk} can be lower bounded by considering the infimum over the inner expectation with respect to estimators $\hat{y}$ that have access to $[Az + \epsilon]_M$ and the support $S^* = \supp(z)$. In this case, the Bayes estimator $\hat{y}^*$ is:
\begin{align*}
    \hat{y}^* &= \E_{z, \epsilon} \left[P_{[d] \setminus M}A_{S^*}[z]_{S^*} \ \bigg\vert\ P_M A_{S^*}[z]_{S^*} + [\epsilon]_M \right] \\
    &= P_{[d] \setminus M}A_{S^*} \E_{z, \epsilon} \left[[z]_{S^*} \ \bigg\vert\ P_M A_{S^*}[z]_{S^*} + [\epsilon]_M \right] \numberthis \label{bayesopt}
\end{align*}
Since $[z]_{S^*} \sim \Gauss(0, \sigma_z^2 \Id_k)$, we can explicitly compute the conditional expectation in Equation \eqref{bayesopt}. Indeed, it is just the ridge regression estimator:
\begin{align*}
    \hat{y}^* = P_{[d] \setminus M}A_{S^*} \left(\Lambda_{S^*}^T \Lambda_{S^*} + \frac{1}{\sigma_z^2} \Id_k\right)^{-1} \Lambda_{S^*}^T (P_M A_{S^*}[z]_{S^*} + [\epsilon]_M) \numberthis \label{ridgereg}
\end{align*}
Where we have set $\Lambda_{S^*} = P_M A_{S^*}$ above to keep notation manageable. Thus, putting all of the above together we have:
\begin{align*}
    L_{mask}(B, k, M) \geq \mathcal{R}_M(\hat{y}^*) + (d - \abs{M})\sigma^2 \numberthis \label{risklb}
\end{align*}
Now let $\hat{y}_{LS}$ be the least squares estimator with access to the support $\supp(z)$:
\begin{align*}
    \hat{y}_{LS} = P_{[d] \setminus M}A_{S^*} \Lambda_{S^*}^+ (P_M A_{S^*}[z]_{S^*} + [\epsilon]_M) \numberthis \label{lsq}
\end{align*}
Then we have $\mathcal{R}_M(\hat{y}_{LS}) \to \mathcal{R}_M(\hat{y}^*)$ as $\sigma_z^2 \to \infty$. If we can now show that $\mathcal{R}_M(P_{[d] \setminus M} A \hat{z}) \to \mathcal{R}_M(\hat{y}_{LS})$, then we will be done by Equation \eqref{risklb}.

Showing this essentially boils down to controlling the error of $\hat{z}$ when OMP fails to recover the true support $S^*$ (because when it recovers the true support, $P_{[d] \setminus M} A \hat{z}$ is exactly $\hat{y}_{LS}$). We do this by appealing to Theorem \ref{omperrorthm}.

Recall that $[\hat{z}]_S = \Lambda_S^+ [y]_M$, where $\Lambda_S = P_M A_S$ with $S = \supp(\hat{z})$ being the support predicted by OMP. Letting $\hat{z}_{LS}$ be the vector whose non-zero components correspond to $[\hat{z}_{LS}]_{S^*} = \Lambda_{S^*}^+ [y]_M$, it will suffice to show $\norm{\hat{z}_{LS} - \hat{z}}^2 \to 0$ as $\sigma_z \to \infty$, since then we will be done due to the fact that $\norm{P_{[d] \setminus M} A}_{op}$ is constant with respect to $\sigma_z$.

Now letting $\Tilde{z} = \hat{z} - \Lambda_S^+ [\epsilon]_M$ (i.e. $\tilde{z}$ represents the part of the signal $z$ recovered by OMP), we have:
\begin{align*}
    \norm{\hat{z}_{LS} - \hat{z}}^2 &= \norm{z - \tilde{z} + (\Lambda_{S^*}^+ - \Lambda_S^+)[\epsilon]_M}^2 \\
    &\leq \norm{z - \tilde{z}}^2 + \norm{(\Lambda_{S^*}^+ - \Lambda_S^+)[\epsilon]_M}^2 + 2\norm{z - \tilde{z}}\norm{(\Lambda_{S^*}^+ - \Lambda_S^+)[\epsilon]_M} \numberthis \label{zlsompgap}
\end{align*}

We begin by first analyzing $\norm{z - \tilde{z}}^2$. To do so, we introduce the notation $[z]_{U, 0}$ to represent the vector in $\R^k$ whose non-zero entries correspond to $[z]_U$ for $U \subset [d], \ \abs{U} \leq k$. Then we can make use of the following decomposition of $\Lambda_{S^*} [z]_{S^*}$:
\begin{align*}
    \Lambda_{S^*} [z]_{S^*} = \Lambda_{S} [z]_{S^* \cap S, 0} + \Lambda_{S^*} [z]_{S^* \setminus S, 0} \numberthis \label{lscancel}
\end{align*}

From Equation \eqref{lscancel} we get:
\begin{align*}
    \norm{z - \tilde{z}}^2 &= \norm{[z]_{S^* \setminus S, 0} - \Lambda_S^+ \Lambda_{S^*} [z]_{S^* \setminus S, 0}}^2 \\
    &\leq \norm{[z]_{S^* \setminus S, 0}}^2 + \norm{(\Lambda_S^T \Lambda_S)^{-1} \Lambda_S^T \Lambda_{S^*} [z]_{S^* \setminus S, 0}}^2 + 2\max\left(\norm{[z]_{S^* \setminus S, 0}}^2, \norm{\Lambda_S^+ \Lambda_{S^*} [z]_{S^* \setminus S, 0}}^2\right) \\
    &= \sum_{i \in S^* \setminus S} O(z_i^2) \numberthis \label{zgap}
\end{align*}
where we passed from the penultimate to the last line by using the $\mu$-incoherence of $P_M A$ to control the middle term in the bound. With Equation \eqref{zgap} in hand, we are finally in a position to apply Theorem \ref{omperrorthm}. Let $\eta = C' \log \sigma_z$ for a sufficiently large constant $C'$. Now for convenience we define:
\begin{align*}
    \gamma = \frac{2\sigma \sqrt{k} \sqrt{2(1 + \eta)\log p}}{1 - (2k - 1)\mu} \numberthis \label{gamdef}
\end{align*}
which corresponds to the lower bound in Equation \eqref{suppset}. Using Theorem \ref{omperrorthm} with Equation \eqref{zgap} we obtain:
\begin{align*}
    \E_{z, \epsilon} \left[\norm{z - \tilde{z}}^2\right] &\leq \sum_{i \in S^*} \Prob(\{\abs{z_i} \geq \gamma\} \cap \{i \notin S\}) O(\E[z_i^2]) + \Prob(\abs{z_i} < \gamma) O(\gamma^2) \\
    &= \sum_{i \in S^*} O\left(\frac{\sigma_z^2}{\sigma_z^{C'} \sqrt{\log p}} + \frac{\gamma^3}{\sigma_z}\right) \numberthis \label{finalzz}
\end{align*}
And clearly Equation \eqref{finalzz} goes to 0 as $\sigma_z \to \infty$. We can apply similar analysis techniques to the term $\norm{(\Lambda_{S^*}^+ - \Lambda_S^+)[\epsilon]_M}^2$ in Equation \eqref{zlsompgap} as well, but for this term we can afford to be less precise. 

Namely, when $S = S^*$, this term is 0. The probability that $S \neq S^*$ can be bounded as:
\begin{align*}
    \Prob(S \neq S^*) &\leq O\left(\frac{k}{\sigma_z^{C'} \sqrt{\log p}}\right) \Prob(\min_{i \in S^*} \abs{z_i} \geq \gamma) \\
    &= O\left(\frac{k \gamma^k}{\sigma_z^{C' + k} \sqrt{\log p}}\right) \numberthis \label{epsprob}
\end{align*}
where again above we used the naive bound for $\Prob(\min_{i \in S^*} \abs{z_i} \geq \gamma)$ (i.e. replacing the density with 1 and integrating from $\pm \gamma$). Now we have:
\begin{align*}
    \E_{z, \epsilon} \left[\norm{(\Lambda_{S^*}^+ - \Lambda_S^+)[\epsilon]_M}^2\right] &\leq \Prob(S \neq S^*) \E_{z, \epsilon}\left[\norm{\Lambda_{S^*}^+ [\epsilon]_M}^2 + \norm{\Lambda_{S}^+ [\epsilon]_M}^2 + 2\max\left(\norm{\Lambda_{S^*}^+ [\epsilon]_M}^2, \norm{\Lambda_{S}^+ [\epsilon]_M}^2\right)\right] \\
    &\leq \Prob(S \neq S^*) O\left(k\E_{\epsilon} \Bigl[\max_{i \in [p]} (A_i^T \epsilon)^2\Bigr]\right) \\
    &\leq \Prob(S \neq S^*) O(k\sigma^2 \log p) \numberthis \label{epserror}
\end{align*}
Putting together Equations \eqref{finalzz} and \eqref{epserror} shows that Equation \eqref{zlsompgap} goes to 0 as $\sigma_z \to \infty$, which proves the result.
\end{proof}

\subsection{Auxiliary Lemmas}
\begin{lemma}
\label{max-chi-square}
Let $X_1, \cdots, X_n$ be $n$ chi-square random variables with 1 degree of freedom, then
\begin{equation*}
    \E\left[\max_{i\in[n]}X_i\right] = O(\log n).
\end{equation*}
\begin{proof}
We bound the maximum via the moment-generating function.

From Jensen's inequality, for $t\in(0, \frac12)$, we have
\begin{align*}
    \exp\left(t \E\left[\max_{i\in[n]}X_i\right]\right) &\leq\E\left[\exp\left(t \max_{i\in[n]}X_i\right)\right] \leq \sum_{i=1}^n \E\left[e^{tX_i}\right] = n(1-2t)^{-\frac12}.
\end{align*}
Setting $t=\frac13$ gives us
\begin{equation*}
    \E\left[\max_{i\in[n]}X_i\right]\leq 3\log n -\frac32\log\frac13 = O(\log n).
\end{equation*}
\end{proof}
\end{lemma}

\end{document}